\makeatletter
\def\input@path{{latex/}}
\let\arxivbibliographystyle\bibliographystyle
\renewcommand{\bibliographystyle}[1]{\arxivbibliographystyle{latex/#1}}
\AtBeginDocument{%
  \let\arxivbibliography\bibliography
  \renewcommand{\bibliography}[1]{\arxivbibliography{latex/#1}}%
}
\makeatother
\PassOptionsToPackage{table}{xcolor}
\documentclass[11pt]{article}

\usepackage[final]{acl}

\usepackage{times}
\usepackage{latexsym}

\usepackage[T1]{fontenc}
\usepackage[utf8]{inputenc}

\usepackage{microtype}

\usepackage{inconsolata}

\usepackage{graphicx}

\usepackage{hyperref}
\usepackage{url}
\usepackage{booktabs}

\usepackage{lineno}

\definecolor{darkblue}{rgb}{0,0,0.5}
\hypersetup{colorlinks=true, citecolor=darkblue, linkcolor=darkblue, urlcolor=darkblue}

\usepackage{subcaption}
\usepackage{booktabs} % for professional tables

\usepackage{amsmath}
\usepackage{amssymb}
\usepackage{mathtools}
\usepackage{amsthm}
\usepackage{dblfloatfix}
\usepackage{placeins}
\usepackage[capitalize,noabbrev]{cleveref}
\theoremstyle{plain}
\newtheorem{theorem}{Theorem}[section]
\newtheorem{proposition}[theorem]{Proposition}

\theoremstyle{definition}

\theoremstyle{remark}

\usepackage{hyperref}
\usepackage{url}
\usepackage{caption}   
\usepackage{placeins}  
\usepackage{booktabs}
\usepackage{siunitx}
\usepackage{threeparttable}
\usepackage{caption}
\usepackage{makecell}
\usepackage{adjustbox}
\usepackage{ulem} 
\usepackage{listings}
\usepackage{tcolorbox}
\tcbuselibrary{breakable, skins}
\newenvironment{llmframe}[1][Raw Model Output]{%
  \begin{tcolorbox}[
    breakable, enhanced, sharp corners,
    colback=gray!2, colframe=black!40, boxrule=0.5pt,
    left=6pt,right=6pt,top=6pt,bottom=6pt,
    title=#1, fonttitle=\bfseries
  ]%
}{\end{tcolorbox}}

\newcommand{\bd}{\boldsymbol}

\definecolor{LightCyan}{rgb}{0.88,0.95,1} 
\usepackage[textsize=tiny]{todonotes}

\title{GMTS: Gradient Magnitude-based Token Selection Improves RLVR Training for LLM Reasoning}

\author{
  \textbf{Outongyi Lv}\textsuperscript{1*} \quad
  \textbf{Yuanwei Zhang}\textsuperscript{1*} \quad
  \textbf{Xiaoqun Zhang}\textsuperscript{1,2\textdagger} \\
  \textsuperscript{1}School of Mathematical Sciences, Shanghai Jiao Tong University \\
  \textsuperscript{2}Institute of Natural Sciences, Shanghai Jiao Tong University \\
  \texttt{\{harry\_lv,sjtuzyw,xqzhang\}@sjtu.edu.cn}
}

\begin{document}
\maketitle
\begingroup
\renewcommand{\thefootnote}{\fnsymbol{footnote}}
\footnotetext[1]{The first two authors contributed equally and are listed alphabetically.}
\footnotetext[2]{Corresponding author.}
\endgroup
\begin{abstract}

Reinforcement learning (RL), particularly RL with Verifiable Rewards (RLVR), has recently emerged as a central paradigm for enhancing large language models' (LLMs) reasoning abilities, demonstrating remarkable effectiveness across reasoning tasks. Recent studies suggest that high-entropy tokens play an exceptionally important role in model training, since training with only the highest 20\% entropy tokens yields significant performance gains. However, why such high-entropy tokens are beneficial remains insufficiently understood. In this work, we find that although high-entropy tokens within one answer tend to correlate with large gradient magnitude, entropy alone fails to consistently reflect token importance across different answers, considering the variations in the answer-level reward signals. Based on this observation, we introduce the \textbf{G}radient \textbf{M}agnitude-based \textbf{T}oken \textbf{S}election (GMTS) method to quantify token importance, which leverages the entropy–gradient connection to approximate gradient-magnitude rankings for token selection. We find that training on the top 20\% tokens ranked by GMTS consistently outperforms entropy-based token selection across three reasoning domains and various model sizes, suggesting that GMTS provides a more fine-grained estimate of token contribution for RLVR training.\footnote{Our code is available at \url{https://github.com/outongyiLv/GMTS}}
\end{abstract}

\section{Introduction}
Large language models (LLMs) have recently shown remarkable capabilities across diverse domains, such as mathematics, image quality assessment, speech, and multimodality \citep{achiam2023gpt,wu2025visualquality,shen2025vlm,bie2026llada2}. Reinforcement learning (RL) has made substantial contributions to the advancement of LLMs, demonstrating outstanding effectiveness in enhancing performance \citep{ouyang2022training,lee2024rlaif,shinn2023reflexion}. DeepSeek introduced Group Relative Policy Optimization (GRPO) \citep{shao2024deepseekmath}, a method that dispenses with the value model in favor of simple rule-based rewards and adopts group-level advantage estimation, leading to significant improvements on reasoning benchmarks. Based on GRPO, several variants, such as Dynamic sAmpling Policy Optimization (DAPO) \citep{yu2025dapo} and Group Sequence Policy Optimization (GSPO) \citep{zheng2025group}, have been proposed, showing that RLVR provides a scalable and efficient paradigm for advancing the reasoning abilities of LLMs. Nevertheless, these methods adopt a uniform objective function across all tokens, without fully accounting for the fact that different tokens contribute unevenly to RLVR.

\begin{figure*}[h!]
  \centering
 \includegraphics[width=0.93\linewidth]{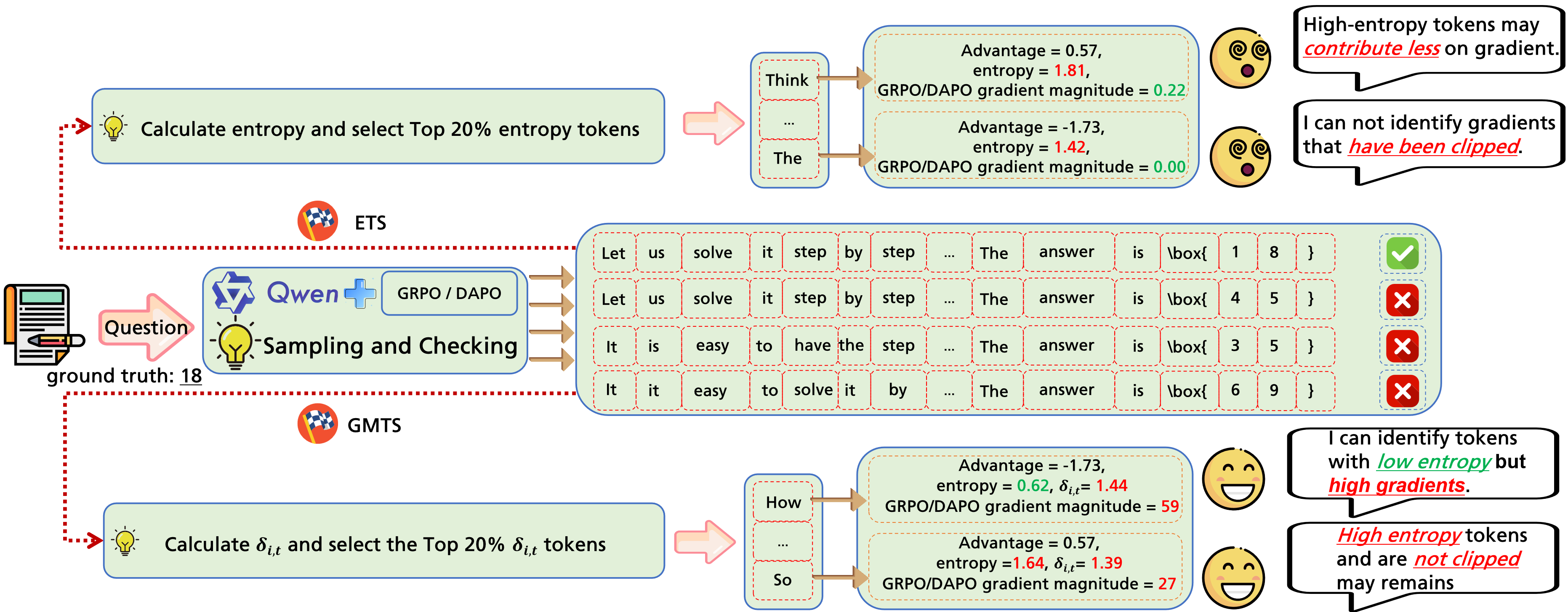} 
  \caption{Overall architecture of the GMTS framework: High-entropy tokens selected by ETS may contribute less to the token gradient or be affected by gradient clipping. By computing $\delta_{i,t}$, GMTS mitigates this problem. Moreover, GMTS can identify low-entropy tokens that are overlooked by ETS but still contribute more on token gradient.}
  \label{fig:overall}
\end{figure*}

Recent studies have focused on distinguishing the token importance by designing token-level advantage functions \citep{yang2025not,liu2026gdpo, ma2026fipo}, introducing new objective functions \citep{cui2025entropy,li2025cure}, and filtering out low importance tokens \citep{wang2025beyond,lv2026tokens}. Concurrently, several studies have also emphasized the central role of high token entropy in enhancing RLVR \citep{wang2025beyond,cui2025entropy,mengsparse}, and controlling entropy provides a way to strengthen the RLVR reasoning ability. In these works, \citet{wang2025beyond} shows that training with only the top 20\% high-entropy tokens yields substantially better performance across the reasoning tasks than training with all tokens.

\begin{center}
 \textit{However, can entropy alone sufficiently indicate token importance?}
\end{center}

In this work, we observe that within a single answer, high-entropy tokens tend to be associated with large gradient magnitude. However, due to variations in reward signals and sample-specific characteristics, entropy alone fails to adequately capture token importance across different answers, whereas gradient magnitude provides a more reliable quantification. Based on this, we propose the \textbf{G}radient \textbf{M}agnitude-based \textbf{T}oken \textbf{S}election (GMTS) method. In contrast to \textbf{E}ntropy-based \textbf{T}oken \textbf{S}election (ETS) \citep{wang2025beyond} under the same settings, GMTS can more effectively distinguish token importance across different answers. Empirically, our experiments show that GMTS consistently improves performance across diverse reasoning tasks and model sizes. For example, GMTS achieves approximately \textbf{1-3} percentage point gains on math reasoning benchmarks. Furthermore, GMTS is also simple to implement and can be easily integrated into existing RLVR frameworks such as GRPO and DAPO. Since all the required components are already available during standard training, GMTS introduces minimal computational overhead. In summary, the contributions of this work are threefold:

\begin{itemize}
    \item \textit{Establish the relation between token entropy and gradient magnitude.}
    We observe that within a single answer, tokens with high entropy correlate strongly with large gradient magnitude, which explains why entropy is an effective indicator of token importance. However, we further demonstrate that this correlation does not hold consistently across different answers when their reward signals vary significantly. This motivates the use of gradient magnitude as a more robust and reliable metric for quantifying token importance.
    
    \item \textit{GMTS for quantifying token importance.} Building upon the relationship between entropy and gradient magnitude, we propose GMTS for quantifying token importance. GMTS effectively retains the benefits of using entropy as a signal while overcoming its key limitations in quantifying token importance across different answers.

    \item \textit{Empirical studies across reasoning domains and model scales.}
    Our empirical studies show that GMTS consistently outperforms the baselines across three reasoning domains and multiple model scales. Detailed experimental results and settings are provided in \cref{experiment}. These results demonstrate the effectiveness of GMTS in improving RLVR performance, as well as its adaptability across different tasks and model scales.

\end{itemize}

\subsection{Related Works}

\paragraph{Token-level advantage estimation.} 
\citet{chen2025seed} leverage the differences between positive and negative advantage-induced token-level signals, together with entropy, to redistribute token-level advantages. \citet{tan2025gtpo} assign entropy-weighted rewards to emphasize high-entropy tokens and down-weight low-entropy ones, mitigating model over-updating. At the same time, \citet{cheng2026reasoning} show that high-entropy tokens correlate with exploration-facilitating behaviors and directly incorporate entropy into advantage estimation. \citet{wang2025emergent} focus on high-impact strategy tokens in reasoning and amplifies their learning signals by modifying the advantages, while \citet{ma2026fipo} use token-level future KL to reassess each token's contribution.

\paragraph{Objective function design.} \citet{wang2025stabilizing} introduce an entropy-aware clipping strategy that applies stricter constraints to low-entropy tokens while relaxing updates for high-entropy ones. \citet{cui2025entropy} observe that decreasing entropy can lead to entropy collapse and address this issue by dynamically adjusting the entropy regularization strength. Building on this direction, \citet{li2025cure} identify critical positions using high-entropy tokens, resample them to construct branched trajectories, and jointly train RLVR on both the original and branched trajectories. In addition, \citet{yao2025diversity} systematically study diversity in RLVR reasoning and incorporate diversity-aware training into the RLVR objective.

\paragraph{Low importance token filtering.} \citet{wang2025beyond}
find that high-entropy tokens often serve as critical forking thinking tokens, and leverage this observation by selecting only the top 20\% entropy tokens for RLVR training, achieving significant performance gains. GMTS also lies in this direction, filtering tokens with low gradient magnitude and using remaining tokens in RLVR training.

\section{Preliminaries}
\subsection{LLM Formulation and Notations}

Let $\mathcal{V}=\{x^1,x^2,\dots,x^V\}$ denote the finite vocabulary of tokens, where $V$ is the size of the vocabulary. Given an input query $\bd{q}$, an LLM parameterized by $\theta$ defines an autoregressive policy $\pi_\theta$ that generates an output sequence $\bd{o}=(o_1,\dots,o_T)$ token by token. At step $t$, conditioned on the query and the previously generated tokens $\bd{o}_{<t}$, the model predicts a distribution over the vocabulary $\pi_\theta(\cdot \mid \bd{q},\bd{o}_{<t})=[p_{t,1},p_{t,2},\dots,p_{t,V}]$, from which the next token $o_t$ is sampled or selected. This process continues until an end-of-sequence token is generated or the maximum generation length is reached. We define the entropy of the predictive distribution at step $t$ as 
$$
E_t := -\sum_{k=1}^{V} p_{t,k}\log p_{t,k}.
$$
This notation will be used throughout our method.

\subsection{RLVR Methods}
This section briefly reviews two representative RLVR methods: GRPO and DAPO.

GRPO \citep{shao2024deepseekmath} is a PPO-style method \citep{schulman2017proximal} that estimates advantages from a group of sampled responses, thereby removing the need for a separate value model. Given a query $\bd{q}\sim \mathcal{D}$, GRPO samples $G$ answers $[\bd{o}_1, \bd{o}_2, \ldots, \bd{o}_G]$ from $\pi_{\text{old}}$, with rewards $R = [R_1, R_2, \ldots, R_G]$ to compute the advantage $A_{i} = \frac{R_i-\text{mean}(R)}{\text{std}(R)} $. This response-level advantage is assigned to all tokens in $\bd{o}_i$, i.e., $A_{i,t}=A_i$. The optimization objective of GRPO is:
\begin{equation}
    \max_{\theta} \mathbb{E}_{\bd{q}} \left[ \frac{1}{G}\sum_{i=1}^G\frac{1}{|\bd{o}_i|}\sum_{t=1}^{|\bd{o}_i|} \ell_{i, t}(\theta)\right],
\end{equation}
where $\ell_{i,t}(\theta)$ denotes the token-level PPO-style objective:
$$
\begin{aligned}
\ell_{i,t}(\theta):= &\min\big(\text{clip}(r_{i, t}(\theta), 1-\epsilon_1, 1+\epsilon_2)A_{i, t},\\
&r_{i,t}(\theta) A_{i,t}\big)-\beta \cdot\mathbb{D}^{KL}_{i,t}(\theta).
\end{aligned}
$$
In this expression, $\mathbb{D}^{KL}_{i,t}(\theta)$ denotes the KL divergence between the current and reference predictive distributions $\pi_{\theta}(\cdot \mid q,\boldsymbol{o}_{<t})$ and $\pi_{\text{ref}}(\cdot \mid q,\boldsymbol{o}_{<t})$, $\pi_{\text{ref}}$ is the reference policy and $r_{i, t}(\theta) = \pi_{\theta}(o_{i,t}|\bd{q}, \bd{o}_{i, <t})/\pi_{\text{old}}(o_{i,t}|\bd{q}, \bd{o}_{i, <t})$. $\beta$ is the parameter for the KL penalty term and $
\epsilon_1, \epsilon_2$ are clipping hyperparameters. In practice, following \citet{schulman2020approximating}, we estimate the per-token KL penalty using a single-sample unbiased estimator.

DAPO builds on GRPO with several practical modifications \citep{yu2025dapo}. Compared with GRPO, DAPO removes the KL penalty from the objective and adopts dynamic sampling to discard uninformative groups in which all responses are either correct or incorrect. It also uses token-level loss averaging and a larger upper clipping threshold $\epsilon_2$ to improve training stability.

\section{Methodology}
\begin{figure*}[t]
  \centering
    \includegraphics[width=0.99\linewidth]{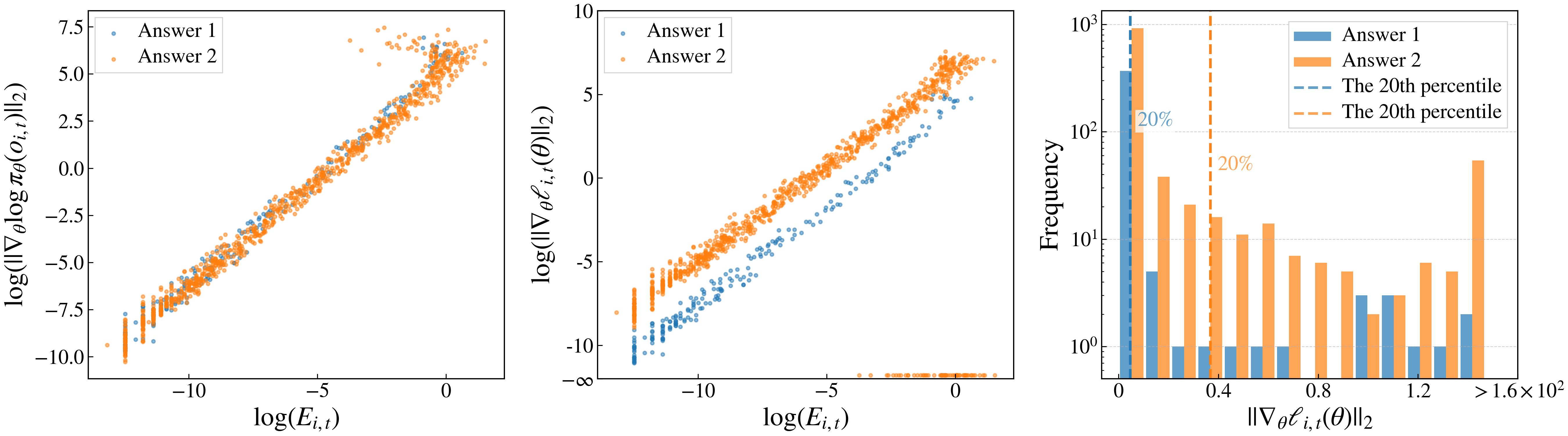} 
  \caption{We use DAPO to train Qwen2.5-math-1.5B on MATH-$12$K dataset under the setting of $G=16$. At training step $100$, we selected two answers with different advantages in one question group, and computed the gradients $\nabla_{\theta} \ell_{i,t}(\theta)$, $\nabla_{\theta} \log \pi_{\theta}(o_{i,t})$, and entropy $E_{i,t}$ of all the tokens. With these collected data, we plotted entropy against the token-level log-probability gradients under the log scale \textbf{(left)}, entropy against the true gradient magnitude under the log scale \textbf{(middle)}, and the distributions of true gradient magnitude together with their 20th percentile boundaries \textbf{(right)}.}
  \label{fig:method}
\end{figure*}

\begin{figure*}[t]
  \centering
    \includegraphics[width=0.99\linewidth]{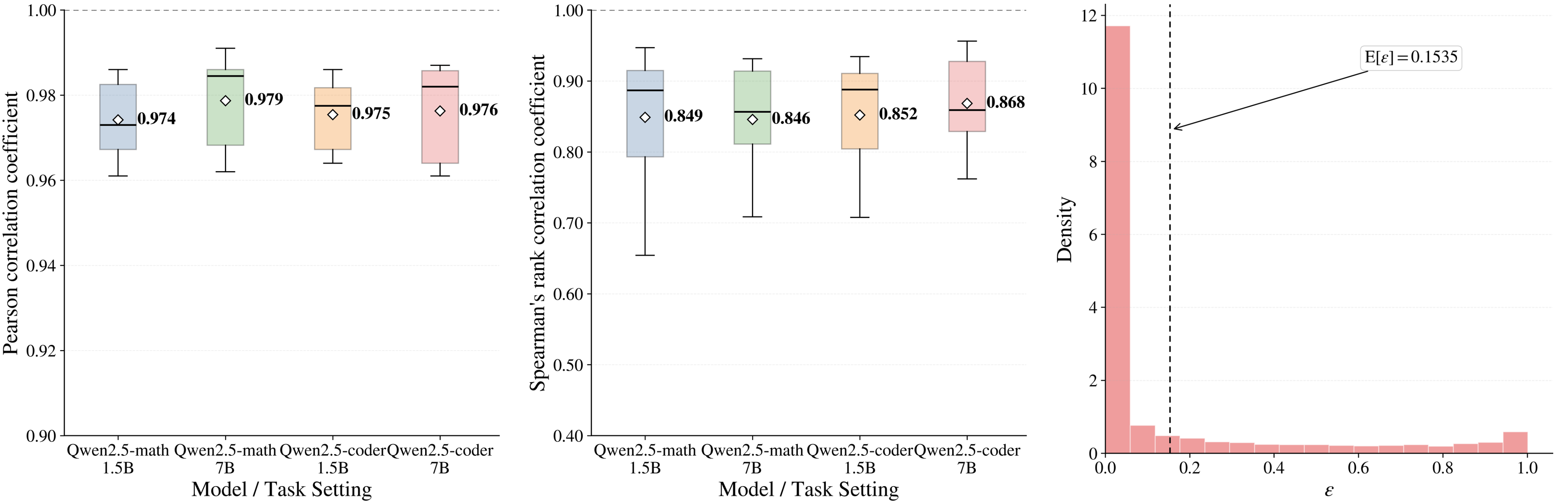} 
  % \caption{\textcolor{blue}{We selected 10 questions from MATH-500 / Kod-Code using the Qwen2.5-math / Qwen2.5-coder at the 1.5B and 7B scales, respectively. For each question, we generated 16 responses and computed, within the response group of the same question, the Pearson correlation coefficient between $\log(E_{i,t})$ and $\log(\|\nabla_{\theta} \log \pi_{\theta}(o_{i,t})\|_2)$ \textbf{(left)}, and Spearman's rank correlation coefficient between $|E_{i,t}\cdot \omega_{i,t}(\theta)|$ and $\|\nabla_{\theta} \ell_{i,t}(\theta)\|_2$ \textbf{(middle)}. We further plot the distribution of $\varepsilon$ over all 160 generated responses on the MATH-500 benchmark \textbf{(right)}}.}
  \caption{We select 10 questions from each of MATH-500 and KodCode, and evaluate Qwen2.5-Math and Qwen2.5-Coder at the 1.5B and 7B scales on the math and code benchmarks, respectively. For each question, we generate 16 responses. Within each group of responses to the same question, we compute the Pearson correlation coefficient between $\log(E_{i,t})$ and $\log\bigl(\|\nabla_{\theta} \log \pi_{\theta}(o_{i,t})\|_2\bigr)$ \textbf{(left)}, and the Spearman rank correlation coefficient between the GMTS score $|E_{i,t}\cdot \omega_{i,t}(\theta)|$ and the true gradient magnitude $\|\nabla_{\theta} \ell_{i,t}(\theta)\|_2$ \textbf{(middle)}. We further plot the distribution of $\varepsilon$ over all 160 generated responses on MATH-500 \textbf{(right)}.}
  \label{fig:method_add2}
\end{figure*}

However, both GRPO and DAPO treat tokens in the same response equally in training objectives by assigning the same response-level advantage to all tokens, thereby ignoring their unequal contributions to policy updates. In this section, we first analyze token-level differences from the perspective of gradient magnitude and connect them to token entropy. We then introduce GMTS, a gradient-magnitude-based token selection method for RLVR training.

\paragraph{Gradient of each token.}
Following \citet{yang2025not}, the gradient of the token-level objective $
\ell_{i,t}(\theta)$ can be expressed as a scalar coefficient times the score function:
$$
\nabla_\theta \ell_{i,t}(\theta) =  \omega_{i,t}(\theta) \cdot\nabla_\theta \log \pi_\theta (o_{i,t}),
$$
where we use $\pi_\theta (o_{i,t})$ to represent $\pi_{\theta}(o_{i,t}|\bd{q}, \bd{o}_{i,<t})$ for ease of exposition. The coefficient  $\omega_{i,t}(\theta)$ is given by
$$
\begin{aligned}
\omega_{i,t}(\theta) &= r_{i, t}(\theta) A_{i,t}\cdot \mathbb{I}_{\epsilon_1, \epsilon_2}(r_{i,t}(\theta),A_{i,t})\\
&+\beta \frac{\pi_{\text{ref}}(o_{i,t})}{\pi_{\theta}(o_{i,t})}-\beta.
\end{aligned}
$$
Here, $\mathbb{I}_{\epsilon_1, \epsilon_2}(\cdot,\cdot)$ denotes the clipping indicator function:
\begin{equation*}  
\mathbb{I}_{\epsilon_1, \epsilon_2}(a,b) = 
\left\{  
     \begin{array}{ll} 
     0 & \text{if } b > 0,\  a > 1+\epsilon_2, \\  
     0 & \text{if } b < 0,\  a < 1-\epsilon_1, \\
     1 & \text{otherwise}. 
     \end{array}
\right.  
\end{equation*}
The decomposition of $\nabla_\theta \ell_{i,t}(\theta)$ shows that $\omega_{i,t}(\theta)$ collects the effective
token-level learning signals in RLVR, including the advantage value, clipping, and KL
regularization terms. Therefore, it directly scales the policy gradient at each
token and determines how strongly that token contributes to the overall parameter
update. In DAPO, the KL term is omitted, so $\omega_{i,t}(\theta)$ reduces to
$r_{i,t}(\theta)A_{i,t}
\mathbb{I}_{\epsilon_1,\epsilon_2}(r_{i,t}(\theta),A_{i,t})$.
When $\pi_\theta$ remains close to $\pi_{\rm old}$ and the clipping constraint is
inactive, we have $r_{i,t}(\theta)\approx 1$ and hence
$\omega_{i,t}(\theta)$ is approximately $A_{i,t}$.

\paragraph{Token gradient magnitude and entropy.} 
As illustrated in the left panel of \cref{fig:method}, the magnitude of the token-level log-probability gradient 
$\nabla_{\theta}\log \pi_{\theta}(o_{i,t}|\bd{q}, \bd{o}_{<t})$ exhibits an approximately linear relationship with token entropy when both quantities are plotted on a logarithmic scale. 
To better understand this phenomenon, we analyze the gradient structure in detail. Consider a query input $\bd{q}$ and an output sequence $\bd{o}$, and focus on the $t$-th generated token $o_t$. 
Let $\bd{z}_t = [z_{t,1}, z_{t,2}, \dots, z_{t, V}]$ denote the logits produced by the final layer of the LLM at time step $t$, and let
$\pi_{\theta}(\cdot\mid \bd{q}, \bd{o}_{<t}) = \operatorname{Softmax}(\bd{z}_t) := [p_{t,1}, p_{t,2}, \dots, p_{t,V}]$
be the resulting predictive distribution over the vocabulary. 
\begin{proposition}
\label{theorem1}
Let $G_{t}$ be the $\ell_2$-norm of the gradient of the log-probability $\log \pi_{\theta}(o_t|\bd{q}, \bd{o}_{<t}) $ with respect to the logits $z_t$, $E_t$ is the entropy of distribution $\pi_{\theta}(\cdot|\bd{q}, \bd{o}_{<t})$. Let $1-\pi_{\theta}(o_{t}|\bd{q}, \bd{o}_{<t})=\varepsilon$, then we have:
    $$\lim_{\varepsilon\rightarrow0} \frac{\log G_{t}}{\log E_t} = 1.$$
\end{proposition}
The detailed derivation for \cref{theorem1} is provided in \cref{app: derivation}. It establishes a direct connection between token entropy and the magnitude of the log-probability gradient with respect to the logits. As shown by the distribution of $\varepsilon$ on the right panel of \cref{fig:method_add2}, most tokens have $\varepsilon$ values close to $0$. This provides a theoretical explanation for the linearity observed in \cref{fig:method}, as well as its slope being close to $1$. Furthermore, let $\frac{\partial \bd{z}_t}{\partial \theta}$ denote the Jacobian of the logits with respect to the model parameters. By the chain rule $\nabla_{\theta} \log \pi_{\theta}(o_t\mid \bd{q}, \bd{o}_{<t})
= \frac{\partial \bd{z}_t}{\partial \theta}
\cdot \nabla_{\bd{z}_t} \log \pi_{\theta}(o_t\mid \bd{q}, \bd{o}_{<t})$, which implies that tokens with higher entropy tend to induce larger gradients with respect to the model parameters.
This analysis provides a principled explanation for the empirical observation that high-entropy tokens may contribute more significantly to parameter updates in RLVR training \citep{wang2025beyond}.

\paragraph{Gradient magnitude token selection.}
However, the coefficients $\omega_{i,t}(\theta)$ in the token gradients can exhibit substantial variability and are particularly influenced by the clipping constraints and the advantage values derived from answer-level reward signals. As shown in the middle and right panels of \cref{fig:method}, tokens belonging to answers with different advantages display notable differences in gradient magnitudes, along with distinct percentile distributions. 
These observations suggest that entropy alone is insufficient to reflect a token's contribution in RLVR training. To address this limitation, we propose the GMTS method, which assesses token importance directly from the gradient magnitude. Since directly computing the true gradient norm is computationally intractable during RL training, we approximate $\nabla_{\theta}\log \pi_{\theta}(o_{i,t}|\bd{q}, \bd{o}_{i,<t})$ using token entropy. And we define the GMTS score of token $o_{i,t}$ as: 
$$\delta_{i, t} = |E_{i, t}\cdot \omega_{i,t}(\theta)|,$$
which serves as a tractable measure of token importance. Based on that, the training objective of GMTS is as follows: 
\begin{equation}
\max_{\theta} \mathbb{E}_{\bd{q}} \left[ \frac{1}{S_{\rho}}\sum_{i=1}^G\sum_{t=1}^{|\bd{o}_i|} {\mathbb{I}[\delta_{i, t}\geq \tau_{\rho}]} \cdot\ell_{i, t}(\theta)\right].
\end{equation}
where $\mathbb{I}(\cdot)$ is the indicator function that evaluates to $1$ if the
condition inside holds and $0$ otherwise. $\rho$ is a predefined ratio specifying the top proportion to be selected, and $\tau_{\rho}$ is the corresponding threshold such that only tokens with $\delta_{i, t}(\theta)\geq \tau_{\rho}$, comprising the top-$\rho$ fraction of all tokens in the batch, are used to compute the gradient, $S_{\rho}$ denotes the number of selected tokens in the group. Thus the loss is averaged only over the selected tokens.

From a computational perspective, calculating $\omega_{i,t}(\theta)$ involves the advantage signal $A_{i,t}$ and the predicted probability from $\pi_{\theta},\pi_{\text{old}}$ and $\pi_{\text{ref}}$ for each token. These quantities are already available during standard RLVR training. Therefore, the additional computational overhead introduced by GMTS is minimal. Moreover, GMTS can be readily integrated into other RLVR training frameworks, with the only difference being the specific formulation of the coefficient $\omega_{i,t}(\theta)$. By leveraging only a small subset of influential tokens for RLVR updates, GMTS offers the potential for more efficient RLVR training.

% \yuanwei{[Discussion in three directions: Advantages, Clipping, KL]}
% \textbf{Comparison with ETS.} Both of these two methods retain only the top-$\rho$ fraction of tokens for policy update. The key difference lies in how token importance is defined and how the selection is performed. 
% As shown in the middle panels of \cref{fig:method}, tokens from different answers with similar entropy can have distinct gradient magnitude. In particular, under PPO-style objectives with ratio clipping, tokens whose probability ratios fall outside the clipping region contribute zero (or near-zero) gradient. However, ETS ranks tokens solely by entropy, which may select tokens that are ultimately inactive in the objective. In contrast, GMTS uses a score that is tied to the effective training signal (through the clipping/weighting term), and therefore prioritizes tokens that actually contribute non-zero gradients to the objective.

\paragraph{Comparison with ETS.}
Both ETS and GMTS retain only the top-$\rho$ fraction of tokens for policy updates, but they differ in how token importance is measured. ETS ranks tokens solely by entropy, whereas GMTS scores each token by $\delta_{i,t}$, which incorporate the effective scalar coefficient in the token gradient. This difference matters in three aspects: (1) Tokens with similar entropy can have very different effects when they belong to responses with different advantages. As shown in the middle panels of \cref{fig:method}, the same level of entropy may correspond to distinct gradient magnitudes due to different response-level reward signals. (2) PPO-style clipping can suppress the policy-gradient term when the probability ratio falls outside the clipping range. In such cases, ETS may still select high-entropy tokens whose gradients are inactive or largely reduced, while GMTS downweights them through the clipping indicator in $\omega_{i,t}(\theta)$. (3) When KL regularization is used, the effective gradient is also affected by the reference-policy correction term, which is not reflected by entropy alone. Thus, GMTS prioritizes tokens that contribute more directly to the actual RLVR update.

% \loty{To more clearly illustrate the differences between ETS and GMTS, we visualize our method in Fig.~\ref{fig:overall}. In the following Section~\ref{experiment}, we empirically demonstrate the advantages of GMTS.} 

% \begin{enumerate}
    
%     \item Tokens with similar entropy can have very different effects when they belong to responses with different advantages. As shown in the middle panels of \cref{fig:method}, the same level of entropy may correspond to distinct gradient magnitudes due to different response-level reward signals. 
%     \item PPO-style clipping can suppress the policy-gradient term when the probability ratio falls outside the clipping range. In such cases, ETS may still select high-entropy tokens whose gradients are inactive or largely reduced, while GMTS downweights them through the clipping indicator in $\omega_{i,t}(\theta)$.
%     \item When KL regularization is used, the effective gradient is also affected by the reference-policy correction term, which is not reflected by entropy alone. Thus, GMTS prioritizes tokens that contribute more directly to the actual RLVR update.
% \end{enumerate}

% \newpage

\section{Numerical Experiments}
\label{experiment}
\subsection{Training and Evaluation Settings}

We implement and evaluate GMTS and ETS within the verl framework\footnote{\url{https://github.com/volcengine/verl}} across GRPO and DAPO. Furthermore, we provide a standalone implementation\footnote{\url{https://github.com/policy-gradient/GRPO-Zero}} that not only gives a clear understanding of GMTS but also supports training under limited GPUs. In our experiments, unless otherwise specified, both ETS and GMTS are trained using the \textbf{Top 20\%} of tokens, consistent with the selection ratio recommended by ETS \citep{wang2025beyond}.

\noindent\textbf{Training Details.} 
We evaluate GMTS across three distinct domains: \textbf{MATH}, \textbf{CODE}, and \textbf{COMMONSENSE (CS)}. Primary results are reported for the 1.5B and 7B model variants, with additional experiments conducted on the 8B model specifically for \textbf{MATH} domain. We use the following training sets: For MATH domain, we utilize the MATH-12K \citep{lightman2023let} dataset (1.5B and 7B) and DAPO-MATH-17K \citep{wang2025beyond} (8B). For CODE domain, we utilize the KodCode \citep{xu2025kodcode}. For CS domain, we utilize the CS-QA \citep{talmor2019commonsenseqa}. Other details for training on these tasks are all summarized in \cref{app: training_details}.

\noindent\textbf{Evaluation Details.} For \textbf{MATH} domain, we evaluate the 1.5B and 7B models on five benchmarks: AIME2024, AMC23, MATH-500, Minerva, and OlympiadBench. For the 8B model, we additionally report results on AIME2025, since it is particularly challenging for small models (1.5B and 7B). For \textbf{CODE} domain, we utilize standard benchmarks widely adopted in the Qwen3-Coder evaluation \footnote{\url{https://github.com/QwenLM/Qwen3-Coder/}}: LiveCodeBench (202407 – 202411), MBPP, HumanEval, and Bigcode-Bench. For \textbf{CS} domain, we utilize the test of CS-QA and more general CS-QA2 \citep{talmor2022commonsenseqa}. For \textbf{MATH} and \textbf{CS} domains, we generate 16 candidate answers for each question and report the average accuracy as \textit{average@16}. All generations are performed with the temperature of $T=1.0$. For \textbf{CODE} domain, we employ the greedy strategy for evaluation.

\begin{figure*}[h!]
  \centering
    \includegraphics[width=0.95\linewidth]{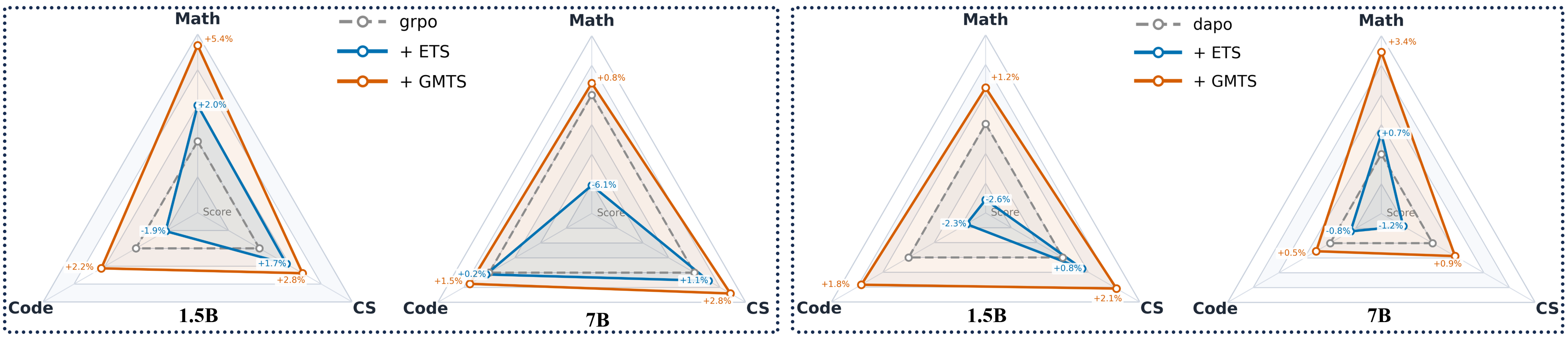} 
    \caption{Overall performance comparison. GMTS consistently outperforms GRPO and DAPO baselines as well as the ETS method across multiple model scales (1.5B and 7B).}
  \label{fig:overall_experiment}
\end{figure*}

\subsection{Main Results}
In \cref{fig:overall_experiment}, we illustrate the global performance of ETS and GMTS across different domains. Next, we show more details of these results on each domain.

\subsubsection{Results on MATH Domain}
We implement ETS/GMTS on GRPO/DAPO, and apply them to the Qwen2.5-math-1.5B and 7B model. All the detailed results are in \cref{tab:results_q25_1p5b,tab:results_q25_7b}. 

\begin{table*}[h!]
  \centering
  \caption{Evaluation on 5 math-reasoning benchmarks (ignore AIME2025) for ETS/GMTS with DAPO/GRPO on \textbf{Qwen2.5-math-1.5B} under \textit{average@16} accuracy (\%).}
  \label{tab:results_q25_1p5b}
  \begin{adjustbox}{max width=0.88\linewidth}
  \begin{threeparttable}
    \begin{tabular}{l *{6}{S[table-format=2.2]}}
      \toprule
      & \multicolumn{1}{c}{AIME2024} & \multicolumn{1}{c}{AMC23} & \multicolumn{1}{c}{MATH-500} & \multicolumn{1}{c}{Minerva} & \multicolumn{1}{c}{OlympiadBench} & \multicolumn{1}{c}{Avg.} \\
      \cmidrule(lr){2-7}
      
      \textbf{Qwen2.5-math-1.5B} & 2.70 & 18.75 & 23.40 & 6.20 & 14.10 & 13.03 \\
      \midrule
      {DAPO}         & \underline{14.17} & 50.62 &  \textbf{73.85} & 29.47 & \underline{37.17} & \underline{41.06} \\
      {{ + ETS}} & \underline{14.17} & 48.44 & 72.25 & 29.01 & 36.16 & 40.01 \\
      \rowcolor{LightCyan}      
      \textbf{ + GMTS}  &  \textbf{14.58} &  \underline{50.63} & \underline{73.83} &  \textbf{30.51} &  \textbf{38.23} & {\bfseries 41.56} \\
      
      \midrule
      {GRPO}         & 10.63 & 48.12 & 71.59 & 28.81 & 34.90 & 38.81 \\
      { + ETS} &  \underline{12.71} & 48.89 & \underline{72.39} & 28.41 & 35.54 & 39.59 \\
      \rowcolor{LightCyan}
      \textbf{ + GMTS}         & 11.88 & \textbf{53.33} &  \textbf{73.06} &  \textbf{30.11} &  \textbf{36.09} & {\bfseries 40.89} \\
      \bottomrule
    \end{tabular}
  \end{threeparttable}
  \end{adjustbox}
\end{table*}

\begin{table*}[h!]
  \centering
  \caption{Evaluation on 5 math-reasoning benchmarks (ignore AIME2025) for ETS/GMTS with DAPO/GRPO on \textbf{Qwen2.5-math-7B} under \textit{average@16} accuracy (\%).}
  \label{tab:results_q25_7b}
  \begin{adjustbox}{max width=0.88\linewidth}
  \begin{threeparttable}
    \begin{tabular}{l *{6}{S[table-format=2.2]}}
      \toprule
      & \multicolumn{1}{c}{AIME2024} & \multicolumn{1}{c}{AMC23} & \multicolumn{1}{c}{MATH-500} & \multicolumn{1}{c}{Minerva} & \multicolumn{1}{c}{OlympiadBench} & \multicolumn{1}{c}{Avg.} \\
      \cmidrule(lr){2-7}
      \textbf{Qwen2.5-math-7B}                 & 7.29 & 21.09 & 34.45 & 6.09 & 10.92 & 15.97 \\
      \midrule
      DAPO           & 18.12 & 62.50 &  \underline{80.91} & 36.49 &  \underline{44.35} & 48.47 \\
      { + ETS} & \underline{20.00} & \underline{63.75} & 80.04 & 36.72 & 43.56 & 48.81 \\

      \rowcolor{LightCyan}
      \textbf{ + GMTS}           &  \textbf{25.00} &  \textbf{65.16} & 80.40 &  \textbf{37.39} & 42.75 & \textbf{50.14} \\
      
      \midrule
      GRPO           & \textbf{27.08} & 62.34 & 79.76 & 35.18 & \underline{42.87} & \underline{49.45} \\
      { + ETS} & 19.17 & 60.67 & 77.72 & 33.65 & 40.96 & 46.43 \\
      % \textbf{+ Forking Tokens (10\%)} & 23.96 & 60.94 & 79.85 & 36.17 & 41.47 & 48.48 \\
      \rowcolor{LightCyan}
      \textbf{ + GMTS}           & 23.33 &  \textbf{64.83} &  \textbf{81.13} &  \textbf{36.32} &  \textbf{43.60} & {\bfseries 49.84} \\
      % \rowcolor{LightCyan}
      % \textbf{ + GMTS Top (10\%)}           & \underline{23.54} &  \underline{63.91} &  \underline{80.32} &  \underline{36.19} &  42.41 & 49.27 \\
      \bottomrule
    \end{tabular}
  \end{threeparttable}
  \end{adjustbox}
\end{table*}

\begin{table*}[h!]
  \centering
  \caption{Evaluation on 6 math-reasoning benchmarks for ETS/GMTS with DAPO on \textbf{Qwen3-8B} under \textit{average@16} accuracy (\%).}
  \label{tab:results_q3_8b}
  \begin{adjustbox}{max width=0.92\linewidth}
  \begin{threeparttable}
    \begin{tabular}{l *{7}{S[table-format=2.2]}}
      \toprule
      & \multicolumn{1}{c}{AIME2024} & \multicolumn{1}{c}{AIME2025} & \multicolumn{1}{c}{AMC23} & \multicolumn{1}{c}{MATH-500} & \multicolumn{1}{c}{Minerva} & \multicolumn{1}{c}{OlympiadBench} & \multicolumn{1}{c}{Avg.} \\
      \cmidrule(lr){2-8}
      \textbf{Qwen3-8B}                 & 3.75       & 6.04       & 33.91       & 62.99       & 22.29     & 23.66           & 25.44 \\
      \midrule
      DAPO          & 33.33 & 25.42 & 77.81 & 89.24 & 39.77 & 56.67 & 53.71 \\
      { + ETS} & 34.58 & 26.25 & 77.19 & 89.70 & 40.26 & \textbf{57.43} & 54.23 \\
      \rowcolor{LightCyan}
      \textbf{ + GMTS}           & \textbf{39.79} & \textbf{30.00} & \textbf{79.38} & \textbf{91.14} & \textbf{42.42} & 53.76 & \textbf{56.08} \\
      \bottomrule
    \end{tabular}
  \end{threeparttable}
  \end{adjustbox}
\end{table*}

\cref{tab:results_q25_1p5b,tab:results_q25_7b} show that GMTS Top (20\%) consistently outperforms ETS across both DAPO and GRPO backbones (1.5B-DAPO: \textbf{+1.55}, 1.5B-GRPO: \textbf{+1.30}, 7B-DAPO: \textbf{+1.33}, 7B-GRPO: \textbf{+3.41}). To evaluate the performance of GMTS on larger reasoning models, we extended our experiments to Qwen3-8B, using results from \citet{wang2025beyond} as a baseline. As shown in \cref{tab:results_q3_8b}, GMTS consistently delivers notable improvements over ETS across a wide range of math benchmarks, with substantial improvements in challenging tasks such as AIME2024 \textbf{(+5.21)} and AIME2025 \textbf{(+3.75)}. With an overall average gain of \textbf{1.85\%}, these results demonstrate the generality of GMTS, suggesting it has the potential to scale effectively to larger and more complex reasoning models.

\subsubsection{Results on CODE and CS Domains}
We implement ETS/GMTS on GRPO/DAPO, for code domain, we apply them to the Qwen2.5-coder-1.5B and 7B model. For commonsense domain, we apply them to the Qwen2.5-1.5B-base and 7B model. All the detailed results are in \cref{tab:results_qwen_coder_1.5b_grpo,tab:results_qwen_coder_7b_grpo}.

\begin{table*}[t]
  \centering
  \caption{Evaluation on 4 code-reasoning and 2 CS benchmarks for ETS/GMTS with DAPO/GRPO on \textbf{Qwen2.5-coder-1.5B} under \textit{greedy} accuracy (\%) (left) and \textbf{Qwen2.5-base-1.5B} under \textit{average@16} accuracy (\%) (right).}
  \label{tab:results_qwen_coder_1.5b_grpo}
  
  \begin{adjustbox}{max width=0.92\linewidth}
  \begin{threeparttable}
    \begin{tabular}{l *{5}{S[table-format=2.2]}}
      \toprule
      & \multicolumn{1}{c}{LiveCodeBench} & \multicolumn{1}{c}{HumanEval} & \multicolumn{1}{c}{MBPP} & \multicolumn{1}{c}{Bigcode-Bench} & \multicolumn{1}{c}{Avg.} \\
      \cmidrule(lr){2-6}
      \textbf{1.5B Base Model}                 & 1.33       & 22.04       & 12.21       & 4.20       & 9.95     \\
      \midrule
      DAPO          & 11.02 & 72.50 & 73.14 & 26.31 & 45.74  \\
      { + ETS } & 9.00 & 73.51 & 73.12 & 23.22 & 44.71  \\
      \rowcolor{LightCyan}
      \textbf{ + GMTS }           & 11.02 & \textbf{74.72} & \textbf{73.49} & \textbf{27.06} & \textbf{46.58} \\
      \midrule
      
      GRPO          & 10.31 & 72.91 & 74.14 & 25.10 & 45.62  \\
      { + ETS} & 9.00 & 72.92 & 71.34 & 25.71 & 44.74  \\
      \rowcolor{LightCyan}
      \textbf{ + GMTS}           & \textbf{11.02} & \textbf{75.03} & \textbf{74.21} & \textbf{26.33} & \textbf{46.64} \\
      \bottomrule

    \end{tabular}
  \end{threeparttable}

    \begin{threeparttable}
    \begin{tabular}{*{3}{S[table-format=2.2]}}
      \toprule
      \multicolumn{1}{c}{CS-QA} & \multicolumn{1}{c}{CS-QA2} & \multicolumn{1}{c}{Avg.}  \\
      \cmidrule(lr){1-3}
                         10.26       & 4.14      & 7.20  \\
      \midrule
      77.21 & 52.73 & 64.97 \\
      78.09 & 52.82 & 65.46 \\
      \rowcolor{LightCyan}
      \textbf{79.15} & \textbf{53.51} & \textbf{66.33} \\
      
      \midrule
      77.43 & 51.07 & 64.25 \\
      79.82 & 50.92 & 65.37 \\
      \rowcolor{LightCyan}
      \textbf{79.95} & \textbf{52.16} & \textbf{66.04} \\
      \bottomrule 

    \end{tabular}
  \end{threeparttable}
  \end{adjustbox}
\end{table*}

\begin{table*}[t]
  \centering
  \caption{Evaluation on 4 code-reasoning and 2 CS benchmarks for ETS/GMTS with DAPO/GRPO on \textbf{Qwen2.5-coder-7B} under \textit{greedy} accuracy (\%) (left) and \textbf{Qwen2.5-base-7B} under \textit{average@16} accuracy (\%) (right).}
  \label{tab:results_qwen_coder_7b_grpo}
  
  \begin{adjustbox}{max width=0.95\linewidth}
  \begin{threeparttable}
    \begin{tabular}{l *{5}{S[table-format=2.2]}}
      \toprule
      & \multicolumn{1}{c}{LiveCodeBench} & \multicolumn{1}{c}{HumanEval} & \multicolumn{1}{c}{MBPP} & \multicolumn{1}{c}{Bigcode-Bench} & \multicolumn{1}{c}{Avg.} \\
      \cmidrule(lr){2-6}
      \textbf{7B Base Model}                 & 3.20       & 36.33       & 39.92       & 13.50       & 23.24     \\
      \midrule
      DAPO          & 19.94 & 85.10 & 79.02 & 36.84 & 55.20  \\
      { + ETS} & 19.06 & 84.39 & 77.18 & \textbf{38.42} & 54.76  \\
      \rowcolor{LightCyan}
      \textbf{ + GMTS}           & \textbf{20.13} & \textbf{85.62} & \textbf{79.89} & 36.37 & \textbf{55.50} \\
      \midrule
      
      GRPO          & 20.00 & 84.21 & 76.71 & 36.31 & 54.31  \\
      { + ETS} & 17.31 & 85.46 & 77.02 & \textbf{37.98} & 54.44  \\
      \rowcolor{LightCyan}
      \textbf{ + GMTS}           & \textbf{20.60} & \textbf{85.91} & \textbf{77.43} & 36.58 & \textbf{55.13} \\
      \bottomrule 

    \end{tabular}
\end{threeparttable}
\begin{threeparttable}
    \begin{tabular}{*{3}{S[table-format=2.2]}}
      \toprule
      \multicolumn{1}{c}{CS-QA} & \multicolumn{1}{c}{CS-QA2} & \multicolumn{1}{c}{Avg.}  \\
      \cmidrule(lr){1-3}
      31.56       & 14.85      & 23.21  \\
      \midrule
      84.70 & 66.57 & 75.64 \\
      83.44 & 66.09 & 74.77 \\
      
      \rowcolor{LightCyan}
      \textbf{85.65} & \textbf{66.95} & \textbf{76.30} \\
      
      \midrule
      
      84.00 & 62.39 & 73.20 \\
      84.17 & 63.82 & 74.00 \\
      
      \rowcolor{LightCyan}
      \textbf{84.60} & \textbf{65.92} & \textbf{75.26} \\
      
      \bottomrule 

    \end{tabular}
  \end{threeparttable} 
  \end{adjustbox}
\end{table*}
\begin{figure*}[h!]
  \centering
 \includegraphics[width=0.92\linewidth]{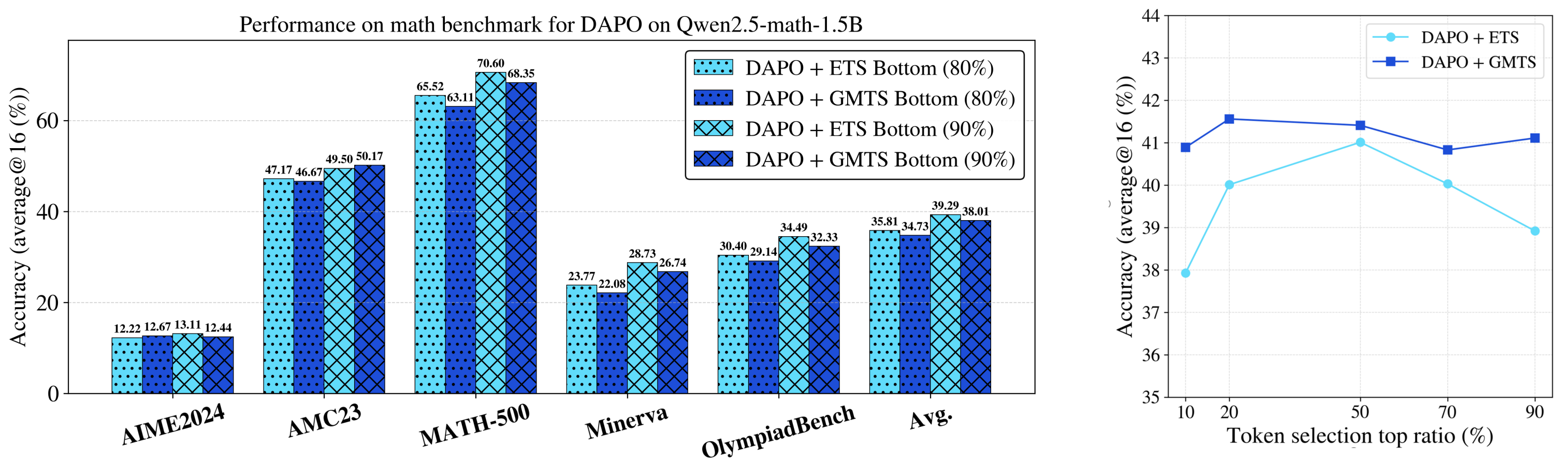} 
  \caption{The \textit{average@16} performance of DAPO, DAPO + ETS Bottom (80\%, 90\%) and DAPO + GMTS Bottom (80\%, 90\%) on Qwen2.5-math-1.5B (\textbf{left}) and the \textit{average@16} overall performance of varying selected ratios for DAPO, DAPO + ETS Top, and DAPO + GMTS Top on Qwen2.5-math-1.5B (\textbf{right}).}
  \label{fig:result_2_new}
\end{figure*}

For HumanEval and MBPP, which consist of Base and Plus subsets, the table reports the average accuracy across both. Similarly, for BigCode-Bench, the reported values represent the mean accuracy across the Full and Hard subsets. As shown in \cref{tab:results_qwen_coder_1.5b_grpo,tab:results_qwen_coder_7b_grpo}, GMTS achieves consistent performance gains within the code domain (1.5B-DAPO: \textbf{+1.87}, 1.5B-GRPO: \textbf{+1.90}, 7B-DAPO: \textbf{+0.74}, 7B-GRPO: \textbf{+0.69}). Furthermore, in the commonsense domain (1.5B-DAPO: \textbf{+0.87}, 1.5B-GRPO: \textbf{+0.67}, 7B-DAPO: \textbf{+1.53}, 7B-GRPO: \textbf{+1.26}). These results validate the generalization of the GMTS method across diverse reasoning tasks.

\subsection{Ablation Study}
\label{ablationstudy}

\noindent\textbf{Linearity and monotonicity consistency.}
We evaluated the linear and monotonic behavior of GMTS using Qwen2.5-math and Qwen2.5-coder at both the 1.5B and 7B scales. As shown in \cref{fig:method_add2} right and middle panels, the linear relationship consistently holds across domains and model scales, indicating that it is not an isolated empirical artifact. Moreover, GMTS exhibits strong monotonicity, suggesting that gradient-magnitude-based selection preserves the ordering of token contributions.

\noindent\textbf{Bottom selection.} 
To examine the role of low gradient magnitude tokens, we follow \citet{wang2025beyond} to evaluate DAPO with bottom selection on Qwen2.5-math-1.5B. As shown in \cref{fig:result_2_new} left and \cref{tab:results_inverse}, GMTS consistently performs worse than ETS, with average drops of \textbf{-1.08} and \textbf{-1.28} at the 80\% and 90\% ratios, respectively. This suggests that low gradient magnitude tokens contribute less to RLVR training, whereas low-entropy tokens may still contain useful signals.

\noindent\textbf{Sensitivity to selection ratios.} To evaluate the impact of selection intensity, we evaluate both DAPO and GRPO on Qwen2.5-math-1.5B with ratios 0.1, 0.2, 0.5, 0.7, and 0.9, with the results reported in \cref{tab:ABL} and \cref{fig:result_2_new} right. It shows that GMTS consistently outperforms ETS in nine out of ten configurations. These findings suggest that the advantages of GMTS are not confined to a specific regime but extend across a broad spectrum of training conditions, providing a more reliable selection mechanism than ETS.

\section{Conclusion}
In this work, we propose GMTS for improving RLVR in LLM reasoning. Although gradient magnitude is correlated with entropy within one answer, entropy alone cannot fully capture token importance across different answers, where the contribution of each token to policy update can vary substantially. GMTS addresses this limitation by leveraging a gradient-magnitude proxy for token selection. Empirically, GMTS achieves consistent improvements over ETS and remains robust across multiple selection ratios, model scales, and evaluation domains. These results provide empirical evidence supporting GMTS as an effective and efficient choice for RLVR.

\section*{Limitations}
% Despite the broad performance gain of GMTS in RLVR, there are several limitations that provide avenues for future research.

% First, the effectiveness of GMTS is supported by extensive empirical evaluations. Although we provide preliminary insights into the relationship between gradient and entropy, a rigorous theoretical justification for the advantage of GMTS over ETS, together with a deeper analytical characterization, remains an important direction for future work.

% Second, due to computational constraints, we are currently unable to evaluate GMTS on larger models, such as 14B and 32B models. Nevertheless, based on the empirical results obtained across different reasoning domains and model scales, we believe that GMTS remains promising for larger-scale models. We leave a more comprehensive evaluation on larger models to future work.

Despite the consistent performance gains of GMTS in RLVR, several limitations remain and point to directions for future research.

First, although extensive empirical evaluations support the effectiveness of GMTS, its theoretical foundation remains only partially understood. In this work, we provide preliminary evidence on the relationship between token entropy and gradient magnitude, but a rigorous theoretical justification for why GMTS outperforms ETS, along with a deeper analytical characterisation of its behaviour, remains an important direction for future work.

Second, due to computational constraints, we have not yet evaluated GMTS on larger models, such as 14B and 32B models. While our empirical results across different reasoning domains and model scales suggest that GMTS has the potential to extend to larger-scale models, a more comprehensive evaluation on such models is left for future work.

\section*{Acknowledgments}

We thank the SJTU supercomputer center for providing the computing services and the anonymous reviewers very much for their careful reading and valuable comments. This work was supported by the National Natural Science Foundation of China (Grant No. 125B2026) and the Natural Science Foundation of Chongqing, China (No. CSTB2023NSCQ-LZX0054).

% \bibliography{colm2026_conference}
% \bibliographystyle{colm2026_conference}

\bibliography{custom}

\appendix
\section{Appendix}
\label{sec:appendix}

\subsection{More Experimental Results}
\label{app: experiments}

We provide additional complete numerical results for \cref{ablationstudy} in the main text, with detailed values reported in \cref{tab:results_inverse,tab:ABL}. The Bottom selection experiment in \cref{tab:results_inverse} is designed to empirically examine whether the top ranked tokens under GMTS contain more training beneficial information than those selected by ETS. In addition, as shown in \cref{tab:ABL}, GMTS clearly outperforms ETS in 9 out of 10 GRPO/DAPO settings, which empirically shows that GMTS generally remains superior to ETS across different top-selection ratios.

\begin{table*}[h!]
  \centering
  \caption{Evaluation on five math-reasoning benchmarks (ignore AIME2025) for ETS/GMTS with DAPO/GRPO on \textbf{Qwen2.5-math-1.5B} under \textit{average@16} accuracy (\%).}
  \label{tab:results_inverse}
  \begin{adjustbox}{max width=\linewidth}
  \begin{threeparttable}
    \begin{tabular}{l *{6}{S[table-format=2.2]}}
      \toprule
      & \multicolumn{1}{c}{AIME2024} & \multicolumn{1}{c}{AMC23} & \multicolumn{1}{c}{MATH-500} & \multicolumn{1}{c}{Minerva} & \multicolumn{1}{c}{OlympiadBench} & \multicolumn{1}{c}{Avg.} \\
      \cmidrule(lr){2-7}
      \textbf{Qwen2.5-math-1.5B} & 2.70 & 18.75 & 23.40 & 6.20 & 14.10 & 13.03 \\
      \midrule
      {DAPO}         & 14.17 & 50.62 &  73.85 & 29.47 & 37.17 & 41.06 \\
      
      {+ ETS Bottom (80\%)} & 12.22 & 47.17 & 65.52 & 23.77 & 30.40 & 35.81 \\
      {+ ETS Bottom (90\%)} & 13.11 & 49.50 & 70.60 & 28.73 & 34.49 & 39.29 \\

\rowcolor{LightCyan}      
      {+ GMTS Bottom (80\%)}           &  12.67 &  46.67 & 63.11 & 22.08 & 29.14 & 34.73 \\
\rowcolor{LightCyan}
      {+ GMTS Bottom (90\%)}           &  12.44 &  50.17 & 68.35 &  26.74 & 32.33 &  38.01 \\
      
      \bottomrule
    \end{tabular}
  \end{threeparttable}
  \end{adjustbox}
\end{table*}

\begin{table*}[h!]
  \centering
  \caption{Evaluation on five math-reasoning benchmarks for ETS/GMTS with DAPO/GRPO on \textbf{Qwen2.5-math-1.5B} under \textit{average@16} accuracy (\%) in different selected ratios.}
  \label{tab:ABL}
  \begin{adjustbox}{max width=\linewidth}
  \begin{threeparttable}
    \begin{tabular}{l *{6}{S[table-format=2.2]}}
      \toprule
      & \multicolumn{1}{c}{AIME2024} & \multicolumn{1}{c}{AMC23} & \multicolumn{1}{c}{MATH-500} & \multicolumn{1}{c}{Minerva} & \multicolumn{1}{c}{OlympiadBench} & \multicolumn{1}{c}{Avg.} \\
      \cmidrule(lr){2-7}
      
      \textbf{Qwen2.5-math-1.5B} & 2.70 & 18.75 & 23.40 & 6.20 & 14.10 & 13.03 \\
      \midrule
      {DAPO}         & 14.17 & 50.62 &  73.85 & 29.47 & 37.17 & 41.06  \\
      \midrule
      { + ETS Top (90\%)} & \textbf{13.96} & 51.72 & 73.97 & \textbf{30.36} & 24.58 & 38.92  \\
\rowcolor{LightCyan}
      \textbf{ + GMTS Top (90\%)}         &  11.04 &  \textbf{52.34} & \textbf{74.20} &  \textbf{30.36} &  \textbf{37.61} & \textbf{41.11}  \\
      \midrule
      { + ETS Top (70\%)} & \textbf{12.08} & 48.44 & \textbf{74.07} & \textbf{29.26} & 36.31 & 40.03  \\
\rowcolor{LightCyan}
      \textbf{ + GMTS Top (70\%)}         &  11.46 &  \textbf{52.50} & 74.00 &  28.80 &  \textbf{37.40} & \textbf{40.83}  \\
      \midrule
      { + ETS Top (50\%)} & 12.50 & 52.81 & 73.22 & 29.55 & \textbf{36.97} & 41.01  \\
\rowcolor{LightCyan}
      \textbf{ + GMTS Top (50\%)}         &  \textbf{13.54} &  \textbf{53.44} & \textbf{73.66} &  \textbf{29.89} &  36.50 & \textbf{41.41}  \\
      \midrule
      { + ETS Top (10\%)} & 12.71 & 46.88 & 70.03 & 27.08 & 32.93 & 37.93  \\
\rowcolor{LightCyan}
      \textbf{ + GMTS Top (10\%)}         &  \textbf{13.33} &  \textbf{52.19} & \textbf{73.53} &  \textbf{29.72} &  \textbf{35.68} &  \textbf{40.89}  \\
      \midrule
      
      \midrule
      {GRPO}         & 10.63 & 48.12 & 71.59 & 28.81 & 34.90 & 38.81 \\
      \midrule

      { + ETS Top (90\%)} &  11.56 & 50.00 & \textbf{72.73} & 29.61 & \textbf{36.07} & 40.00 \\
\rowcolor{LightCyan}
      \textbf{ + GMTS Top (90\%)}         & \textbf{15.33} & \textbf{51.00} &  72.68 &  \textbf{30.51} &  35.91 & \textbf{41.09} \\
      \midrule

      { + ETS Top (70\%)} &  11.88 & 50.62 & 73.51 & 30.31 & 36.61 & 40.59 \\
\rowcolor{LightCyan}
      \textbf{ + GMTS Top (70\%)}         & \textbf{12.92} & \textbf{52.50} &  \textbf{74.22} &  \textbf{30.61} &  \textbf{37.11} & \textbf{41.47} \\
      \midrule

      { + ETS Top (50\%)} &  \textbf{12.92} & \textbf{50.78} & \textbf{73.46} & \textbf{29.25} & \textbf{37.86} & \textbf{40.85} \\
\rowcolor{LightCyan}
      \textbf{ + GMTS Top (50\%)}         & 10.63 & 49.38 &  72.95 &  28.75 &  36.09 & 39.56 \\
      \midrule

      { + ETS Top (10\%)} &  10.63 & 45.94 & 70.00 & 27.30 & 32.97 & 37.37 \\
\rowcolor{LightCyan}
      \textbf{ + GMTS Top (10\%)}         & \textbf{15.00} & \textbf{49.84} &  \textbf{72.24} &  \textbf{29.60} &  \textbf{35.65} & \textbf{40.47} \\

      \bottomrule
    \end{tabular}
  \end{threeparttable}
  \end{adjustbox}
\end{table*}

\subsection{Derivation of token gradient}\label{app: derivation}

\noindent\textbf{Proposition 3.1:}
Let $G_{t}$ be the $\ell_2$-norm of the gradient of the log-probability $\log \pi_{\theta}(o_t|\bd{q}, \bd{o}_{<t}) $ with respect to the logits $z_t$, $E_t$ is the entropy of distribution $\pi_{\theta}(\cdot|\bd{q}, \bd{o}_{<t})$. Let $1-\pi_{\theta}(o_{t}|\bd{q}, \bd{o}_{<t})=\varepsilon$, then we have:
    $$\lim_{\varepsilon\rightarrow0} \frac{\log G_{t}}{\log E_t} = 1.$$

\begin{proof}
Let $p_t=\pi_{\theta}(\cdot\mid \bd{q}, \bd{o}_{<t})$

$$
\begin{aligned}
\nabla_z \log(p_{t,i})
&= \frac{1}{p_{t,i}}\nabla_z p_{t,i}\\
&= \frac{1}{p_{t,i}}\big[p_{t,i}(e_{i}-p_t)\big]\\
&= e_{i} - p_t,
\end{aligned}
$$
where $e_{i}$ is the unit vector whose $i$-th entry is $1$ and all other entries are $0$. Let $\pi_{\theta}(o_{t}|\bd{q}, \bd{o}_{<t})$ = $p_{t,i}$, we can get:
$$
G_t = \sqrt{(1-p_{t,i})^2 + \sum_{j\neq i} p_{t,j}^2}.
$$
According to $1-\pi_{\theta}(o_{t}|\bd{q}, \bd{o}_{<t})=\varepsilon$, then we will have 

$$
\begin{aligned}
    \log G_t &=\log(\sqrt{\varepsilon^2 \left(1+\sum_{j\neq i} q_{t,j}^2\right)})\\&=\log \varepsilon + \log(\sqrt{1+\sum_{j\neq i} q_{t,j}^2}).
\end{aligned}
$$
where $q_{t,k}=\frac{p_{t,k}}{\varepsilon}$.

For entropy $E_t$, we have:
$$
\begin{aligned}
E_t &=  -\sum_{j}p_{t,j} \log p_{t,j} \\ &=-p_{t,i} \log p_{t,i} - (1-p_{t,i})\log(1-p_{t,i})\\ &-(1-p_{t,i}) \sum_{j \neq i} \frac{p_{t,j}}{1-p_{t,i}} \log \frac{p_{t,j}}{1-p_{t,i}}
\end{aligned}
$$
Notice that $-\sum_{j \neq i} \frac{p_{t,j}}{1-p_{t,i}} \log \frac{p_{t,j}}{1-p_{t,i}}$ is the entropy of the renormalized residual distribution after removing the $i$-th probability, we call it $E(p_t/p_{t,i})$, then we have:
$$
E_t = -(1-\varepsilon)\log(1-\varepsilon)-\varepsilon\log\varepsilon + \varepsilon E(p_t/p_{t,i})
$$
Using the Taylor expansion $\log(1-\varepsilon)=-\varepsilon-\frac{\varepsilon^2}{2}+O(\varepsilon^3)$.
We get
$$
\begin{aligned}
&\log E_t\\
&=\log (\varepsilon -\varepsilon \log \varepsilon +  \varepsilon E(p_t/p_{t,i}) -\frac{\varepsilon^2}{2}+O(\varepsilon^3))\\
&=\log\varepsilon + \log(1-\log\varepsilon+ E(p_t/p_{t,i}) -\frac{\varepsilon}{2}+O(\varepsilon^2)).
\end{aligned}
$$
Thus, we have:
$$
\begin{aligned}
&\lim_{\varepsilon\rightarrow0} \frac{\log G_{t}}{\log E_t}  = \\ &\lim_{\varepsilon\rightarrow0} \frac{\log \varepsilon + \log(\sqrt{1+\sum_{j\neq i} q_{t,j}^2})}{\log\varepsilon + \log(1-\log\varepsilon+ E(p_t/p_{t,i}) -\frac{\varepsilon}{2}+O(\varepsilon^2))}\\
&=1.
\end{aligned}
$$
\end{proof}

\subsection{Training Details}\label{app: training_details}

For fairness and reproducibility, we fix the training seed to \textbf{0} in all the experiments, which are performed on eight NVIDIA H100 GPUs, each equipped with 80 GB of memory.

\textbf{Hyperparameter settings for Qwen2.5-math-1.5B and Qwen2.5-math-7B.} 

For the Qwen2.5-math-1.5B and Qwen2.5-math-7B models, the training configurations are as follows: the maximum response length is set to 2048 tokens, and the maximum prompt length to 1024 tokens. Each training batch contained 64 questions, with a mini-batch size of 64. For each question, 16 rollouts are sampled with a temperature of 1.0. The clipping parameters $\epsilon_1$ and $\epsilon_2$ are set to 0.2 and 0.28. The learning rate is fixed at $3\times 10^{-5}$ without applying any learning rate warm-up or scheduling.

\textbf{Hyperparameter settings for Qwen3-8B.} 

Due to computational constraints, we could not fully adopt the settings of \citep{wang2025beyond}, and therefore adjusted the maximum output length. Specifically, we set the maximum response length to 4096 tokens and the maximum prompt length to 1024 tokens. Each training batch contained 512 questions and a mini-batch size of 32, resulting in 16 gradient steps per training batch. For each question, 16 rollouts are sampled with a temperature of 1.0. The clipping parameters $\epsilon_1$ and $\epsilon_2$ are set to 0.2 and 0.28. The learning rate is fixed at $1\times 10^{-6}$ without applying any learning rate warm-up or scheduling.

\textbf{Hyperparameter settings for Qwen2.5-coder-1.5B and Qwen2.5-coder-7B.} 

For the Qwen2.5-coder-1.5B and Qwen2.5-coder-7B models, the training configurations are as follows: the maximum response length is set to 32768 tokens, and the maximum prompt length to 2048 tokens. Each training batch contained 512 questions, with a mini-batch size of 32. For each question, 16 rollouts are sampled with a temperature of 1.0. The clipping parameters $\epsilon_1$ and $\epsilon_2$ are set to 0.2 and 0.28. The learning rate is fixed at $1\times 10^{-6}$ without applying any learning rate warm-up or scheduling.

\textbf{Hyperparameter settings for Qwen2.5-base-1.5B and Qwen2.5-base-7B.} 

For the Qwen2.5-base-1.5B and Qwen2.5-base-7B models, the training configurations are as follows: the maximum response length is set to 4096 tokens, and the maximum prompt length to 1024 tokens. Each training batch contained 512 questions, with a mini-batch size of 32. For each question, 16 rollouts are sampled with a temperature of 1.0. The clipping parameters $\epsilon_1$ and $\epsilon_2$ are set to 0.2 and 0.28. The learning rate is fixed at $1\times 10^{-6}$ without applying any learning rate warm-up or scheduling.

\textbf{Dialogue template}

For the math and CS domains, we adopt the following dialogue template:
\begin{lstlisting}
<|im_start|>system\nPlease reason step by step, and put your final answer within \boxed{}.<|im_end|>\n<|im_start|>user\n"<|im_end|>\n<|im_start|>assistant\n
\end{lstlisting}

For the code domain, we adopt the following dialogue template:
\begin{lstlisting}
<|im_start|>system\nYou are Qwen, created by Alibaba Cloud. You are a helpful assistant.<|im_end|>\n<|im_start|>user\n Please provide a self-contained Python script that solves the following problem:\n {question}\nYou will use the following starter code to write the solution to the problem:\n```python\n{starter_code}\n```\n Please reason step by step, and put your code answer in a markdown code block.
\end{lstlisting}

\subsection{An Example for ETS Top (20\%) and GMTS (20\%) with DAPO on Qwen3-8B under AIME2025}

In DAPO on Qwen3-8B, we selected an example under the \textbf{ETS Top (20\%)} and \textbf{GMTS Top (20\%)} settings to illustrate that the output of GMTS is not only more concise but also more accurate compared to ETS:

\begin{llmframe}[Question and Ground-Truth]
\footnotesize
\textbf{Question}: 
Find the number of ordered pairs $(x,y)$, where both $x$ and $y$ are integers between $-100$ and $100$, inclusive, such that $12x^{2}-xy-6y^{2}=0$.

\textbf{Ground-Truth}:
\[
\boxed{117}
\]

\end{llmframe}

\begin{llmframe}[ETS Top (20\%)]
\footnotesize
Okay, so I need to find the number of ordered pairs $(x, y)$ where both $x$ and $y$ are integers between $-100$ and $100$ inclusive, and they satisfy the equation 
\[
12x^2 - xy - 6y^2 = 0.
\]

Let me think about how to approach this.

First, the equation is quadratic in both $x$ and $y$. Maybe I can factor it or find some relationship between $x$ and $y$. Let me try rearranging the equation. Let me write it again:

\[
12x^2 - xy - 6y^2 = 0.
\]

Hmm. Maybe I can factor this quadratic equation. Let me see. If I treat this as a quadratic in $x$, then maybe I can factor it or use the quadratic formula. Alternatively, maybe factor it as a quadratic in $y$. Let me try both approaches.

Let me try treating it as a quadratic in $x$. So,
\[
12x^2 - yx - 6y^2 = 0.
\]
So, in the form $ax^2 + bx + c = 0$, where $a = 12$, $b = -y$, and $c = -6y^2$. Then, using quadratic formula:

\[
x = \frac{y \pm \sqrt{y^2 - 4 \cdot 12 \cdot (-6y^2)}}{2 \cdot 12}.
\]

Calculating discriminant:

\[
D = y^2 - 4 \cdot 12 \cdot (-6y^2) = y^2 + 288y^2 = 289y^2.
\]

So
\[
\sqrt{D} = \sqrt{289y^2} = 17|y|.
\]
Since $y$ is an integer, $|y|$ is non-negative, so $\sqrt{D}$ is $17|y|$.

Therefore,
\[
x = \frac{y \pm 17|y|}{24}.
\]

---

Case 1: $y \geq 0$. Then $|y| = y$. So
\[
x = \frac{y \pm 17y}{24}.
\]
So two possibilities:
\[
x = \frac{y + 17y}{24} = \frac{18y}{24} = \frac{3y}{4},
\]
or
\[
x = \frac{y - 17y}{24} = \frac{-16y}{24} = -\frac{2y}{3}.
\]

Case 2: $y < 0$. Then $|y| = -y$. So
\[
x = \frac{y \pm 17(-y)}{24}.
\]
Let's compute:

First, with the plus sign:
\[
\frac{y + (-17y)}{24} = \frac{-16y}{24} = -\frac{2y}{3}.
\]

With the minus sign:
\[
\frac{y - (-17y)}{24} = \frac{y + 17y}{24} = \frac{18y}{24} = \frac{3y}{4}.
\]

So regardless of the sign of $y$, the solutions for $x$ are
\[
x = \frac{3y}{4} \quad \text{or} \quad x = -\frac{2y}{3}.
\]

Therefore, for each $y$, $x$ must be either $\tfrac{3y}{4}$ or $-\tfrac{2y}{3}$. But since $x$ has to be an integer, this imposes conditions on $y$. Therefore, $3y/4$ must be integer or $-2y/3$ must be integer. Therefore:

Either $4 \mid 3y$ or $3 \mid 2y$.

But since 3 and 4 are coprime, $4 \mid 3y$ implies $4 \mid y$. Similarly, $3 \mid 2y$ implies $3 \mid y$. Therefore, for $x$ to be integer:

Either $y$ is divisible by $4$, or $y$ is divisible by $3$. Therefore, the solutions are:

- For each $y$ divisible by 4, $x = 3y/4$ is integer.
- For each $y$ divisible by 3, $x = -2y/3$ is integer.

If $y$ is divisible by both 3 and 4 (i.e., by 12), then both $x$'s would be integers.

---

Therefore, the total number of solutions would be the number of $y$'s divisible by 4 plus the number of $y$'s divisible by 3. However, we have to be careful not to double count the cases where $y$ is divisible by both 3 and 4, i.e., by 12. Because those $y$'s would be counted in both sets. Therefore, using inclusion-exclusion, total solutions would be:

\[
N = (\#\{y: 4 \mid y\}) + (\#\{y: 3 \mid y\}) - (\#\{y: 12 \mid y\}).
\]

---

Now we compute.

Range: $-100 \leq y \leq 100$, i.e., 201 integers.

- Number of $y$ divisible by 4: $\lfloor 100/4 \rfloor \times 2 + 1 = 25 \times 2 + 1 = 51$.
- Number of $y$ divisible by 3: $\lfloor 100/3 \rfloor \times 2 + 1 = 33 \times 2 + 1 = 67$.
- Number of $y$ divisible by 12: $\lfloor 100/12 \rfloor \times 2 + 1 = 8 \times 2 + 1 = 17$.

Thus
\[
N = 51 + 67 - 17 = 101.
\]

---

Check special cases:

- $y = 0$. Then the equation is $12x^2 = 0 \implies x=0$. So $(0,0)$ is one solution. This is correctly counted once by the formula.

- $y = 4$. Then $x = 3$. Solution: $(3,4)$ works.

- $y = 3$. Then $x = -2$. Solution: $(-2,3)$ works.

- $y = 12$. Then $x=9$ or $x=-8$, both valid.

All checks out.

---

Finally, also check $x$ range: $-100 \leq x \leq 100$. For $|y| \leq 100$, $x = \tfrac{3y}{4}$ or $x = -\tfrac{2y}{3}$ also lies in $[-100,100]$. Verified.

---

\[
\boxed{101}
\]

\end{llmframe}

\begin{llmframe}[GMTS Top (20\%)]
\footnotesize
Okay, so I need to find the number of ordered pairs $(x, y)$ where both $x$ and $y$ are integers between $-100$ and $100$ inclusive, and they satisfy the equation $12x^2 - xy - 6y^2 = 0$. Hmm, let me think about how to approach this.

First, maybe I can try to factor the equation or simplify it somehow. Let me write down the equation again:

\[
12x^2 - xy - 6y^2 = 0.
\]

Hmm, quadratic in terms of $x$ and $y$. Maybe I can treat this as a quadratic equation in $x$ or $y$? Let me try treating it as a quadratic in $x$. Let me rearrange the terms:

\[
12x^2 - xy - 6y^2 = 0.
\]

If I consider this as a quadratic in $x$, then it's of the form $ax^2 + bx + c = 0$, where:

\[
a = 12, \quad b = -y, \quad c = -6y^2.
\]

Then, using the quadratic formula, 

\[
x = \frac{-b \pm \sqrt{b^2 - 4ac}}{2a} = \frac{y \pm \sqrt{y^2 - 4\cdot 12\cdot (-6y^2)}}{24}.
\]

Simplify inside the square root:

\[
y^2 - 4\cdot 12 \cdot (-6y^2) = y^2 + 288y^2 = 289y^2.
\]

So 
\[
\sqrt{289y^2} = 17|y|.
\]

Therefore,

\[
x = \frac{y \pm 17|y|}{24}.
\]

---

Case 1: $y \geq 0$. Then $|y| = y$. Therefore:

\[
x = \frac{y \pm 17y}{24}.
\]

So two possibilities:

\[
x = \frac{18y}{24} = \frac{3y}{4}, \quad \text{or} \quad x = \frac{-16y}{24} = -\frac{2y}{3}.
\]

Case 2: $y < 0$. Then $|y| = -y$. Therefore:

\[
x = \frac{y \pm 17(-y)}{24}.
\]

So:

First possibility: $y + (-17y) = -16y$, hence $x = -\frac{2y}{3}$.

Second possibility: $y - (-17y) = 18y$, hence $x = \frac{3y}{4}$.

---

So regardless of the sign of $y$, the solutions for $x$ are:

\[
x = \frac{3y}{4} \quad \text{or} \quad x = -\frac{2y}{3}.
\]

Therefore, the solutions to the equation are all pairs $(x,y)$ with $x,y \in [-100,100]$ such that either $x = 3y/4$ or $x = -2y/3$, with integrality conditions.

---

\textbf{Case 1:} $x = 3y/4$. \\
For $x$ to be integer, $y$ must be divisible by 4. Let $y=4k$, then $x=3k$. Range: $-100 \leq 4k \leq 100 \implies -25 \leq k \leq 25$. That gives $25-(-25)+1=51$ solutions.

\textbf{Case 2:} $x = -2y/3$. \\
For $x$ to be integer, $y$ must be divisible by 3. Let $y=3m$, then $x=-2m$. Range: $-100 \leq 3m \leq 100 \implies -33 \leq m \leq 33$. That gives $33-(-33)+1=67$ solutions.

---

\textbf{Overlap check:} \\
Suppose $(x,y)$ satisfies both forms: $(x,y)=(3k,4k)$ and $(x,y)=(-2m,3m)$. Then $3k=-2m$ and $4k=3m$. Solving gives $k=0,m=0$, hence $(0,0)$. So there is exactly one overlap.

---

\textbf{Final count:}
\[
51 + 67 - 1 = 117.
\]

\[
\boxed{117}
\]
\end{llmframe}

\end{document}